 \documentclass[final,5p,times,twocolumn]{elsarticle}

\usepackage{amssymb}

\usepackage{algorithm}
\usepackage{algpseudocode}
\usepackage{graphicx}
\usepackage{lscape}
\usepackage{framed}
\usepackage{url}
\usepackage{color, colortbl}
\usepackage[table]{xcolor}
  \usepackage{soul, color}
  \sethlcolor{green}
\usepackage{comment}
 \usepackage[neveradjust]{paralist}
 \usepackage{amsmath}
\usepackage{amssymb}
\usepackage{amsthm}
 
 \usepackage{multicol}
 
\usepackage{tabularx}

\usepackage{enumitem}
\usepackage{mathtools}
\usepackage{cancel}
\usepackage{centernot}

\usepackage{xspace}
\usepackage{hyphenat}
\usepackage{array}

\usepackage{varwidth}
\usepackage{xcolor}
\usepackage{lipsum}

\usepackage{subcaption}

\newcommand{\union}{\cup}
\newcommand{\intersect}{\cap}
\newcommand{\sementails}{\vdash}
\newcommand{\synentails}{\vDash}
\newcommand{\nsementails}{\nvdash}

\newcommand{\I}{\mathcal{I}}
\newcommand{\signature}{\Sigma}
\newcommand{\signatureC}{\Sigma_C}
\newcommand{\signatureI}{\Sigma_I}
\newcommand{\signatureOP}{\Sigma_\mathit{OP}}
\DeclareMathOperator{\test}{test}
\newcommand{\oent}[1]{\mathtt{#1}}
\newcommand{\oaxiom}[1]{\mathrm{#1}}
\newcommand{\testresult}[1]{\mathrm{#1}}
\newcommand{\entailed}{\testresult{entailed}}
\newcommand{\inconsistent}{\testresult{inconsistent}}
\newcommand{\incoherent}{\testresult{incoherent}}
\newcommand{\absent}{\testresult{absent}}

\algnewcommand{\Input}[1]{\item[\textbf{Input:}] \parbox[t]{\linewidth-1cm}{#1}\vspace{3pt}}
\algnewcommand{\And}{\textbf{and}\xspace}
\algnewcommand{\Or}{\textbf{or}\xspace}
\algnewcommand{\Not}{\textbf{not} }
\algnewcommand{\To}{\textbf{to}\xspace}

\let\oldtextproc\textproc
\renewcommand{\textproc}[1]{\nohyphens{\oldtextproc{#1}}}

\usepackage{soul, color}
\sethlcolor{green}

\newtheorem{definition}{Definition}
\newtheorem{lemma}{Lemma}
\newtheorem{proposition}{Proposition}
\newtheorem{theorem}{Theorem}

\journal{...}

\begin{document}

\begin{frontmatter}



\title{More Effective Ontology Authoring with Test-Driven Development}


   \author[cmk]{C. Maria Keet}
         \author[kd]{Kieren Davies}
    \author[al]{Agnieszka \L{}awrynowicz}

      \address[cmk]{Department of Computer Science, University of Cape Town, South Africa, mkeet@cs.uct.ac.za}
\address[kd]{Department of Computer Science, University of Cape Town, South Africa, kdavies@cs.uct.ac.za}
   \address[al]{Faculty of Computing, Poznan University of Technology, Poland, agnieszka.lawrynowicz@cs.put.poznan.pl}

%

\begin{abstract}
Ontology authoring is a complex process, where commonly the automated reasoner is invoked for verification of newly introduced changes, therewith amounting to a time-consuming test-last approach. 
Test-Driven Development (TDD) for ontology authoring is a recent {\em test-first} approach that aims to reduce authoring time and increase authoring efficiency. Current TDD testing falls short on coverage of OWL features and possible test outcomes, the rigorous foundation thereof, and evaluations to ascertain its effectiveness. 
 We aim to address these issues in one instantiation of TDD for ontology authoring.
We first propose a succinct, logic-based model of TDD testing and present novel TDD algorithms so as to cover also any OWL 2 class expression for the TBox and for the principal ABox assertions, and prove their correctness. The algorithms use methods from the OWL API directly such that reclassification is not necessary for test execution, therewith reducing ontology authoring time. The algorithms were implemented in TDDonto2, a Prot\'eg\'e plugin. 
TDDonto2 was evaluated on editing efficiency and by users. The editing efficiency study demonstrated that it is faster than a typical ontology authoring interface, especially for medium size and large ontologies. The user evaluation demonstrated that modellers make significantly less errors with TDDonto2 compared to the standard Prot\'eg\'e interface and complete their tasks better using less time.
Thus, the results indicate that Test-Driven Development is a promising approach in an ontology development methodology.
\end{abstract}


\end{frontmatter}


\section{Introduction}
\label{sec:intro}

Ontology engineering is facilitated by methods and methodologies, and tooling support for them. The methodologies are mostly information system-like, high-level directions, such as variants on waterfall and lifecycle development \cite{Garcia10,Suarez08}, although more recently, notions of Agile development are being ported to the ontology development setting, e.g., \cite{Blomqvist12,Peroni17}, including testing in some form \cite{Ferre12,GarcaRamos09,Vrandecic06,Warrender15}. Now that most automated reasoners for OWL have become stable and reliable over the years, new methods have been devised that use the reasoner creatively in the support of the ontology authoring process. Notably, the OWL reasoner can also be used for  examining negations \cite{Ferre16,Ferre12}, checking the changes in entailments after an ontology edit \cite{Denaux12,Matentzoglu16}, and proposing compatible object properties for any two classes \cite{KKG13forza}. 

Such tools are motivated at least in part by the time-consuming trial-and-error  authoring process, i.e., where a modeller checks consistency after each edit \cite{Vigo14}. However, aforementioned methods and tools still require classification for each assessment step, which is unsustainable for large or complex ontologies due to prohibitively long classification times. Effectively, these modellers take a {\em test-last} approach to ontology authoring. 
In this respect, ontology engineering methodologies still lag behind software engineering methodologies both with respect to maturity and adoption \cite{Iqbal:Methodologies}.

There are a few recent attempts at explicitly incorporating automated testing with a {\em test-first} approach \cite{KL16,Warrender15}, which is common in software engineering and known under the banner of {\em test-driven development} (TDD) \cite{Beck04}. To the best of our knowledge, there are three tools for TDD unit testing of ontologies \cite{KL16,Warrender15,Scone:Bitbucket} such that one can check whether an axiom is entailed before adding it. They exhibit two shortcomings that prevent potential for wider uptake: 1) certain axioms expressible in OWL 2 DL \cite{OWL2rec} are not supported as TDD tests, such as $\forall R.C \sqsubseteq D$, and 2) the outcome of a test can be only ``pass'' or ``fail'' with no further information about the nature of failure. 
Further, there has been no rigorous theoretical analysis of the techniques used for such ontology testing that avails of the automated reasoner pre-emptively. 
Yet, for modellers to be able to rely on reasoner-driven TDD for ontology authoring---as they do with test-last ontology authoring---such a theoretical foundation would be needed.

In this paper, we aim to fill this gap in rigour and coverage. We first propose a succinct logic-based model of TDD unit testing as a prerequisite. Subsequently, we generalise the piecemeal algorithms of \cite{KL16} to cover also {\em any} OWL 2 class expression in the axiom under test for not only the TBox, as in \cite{KL16}, but also for the principal ABox assertions, and prove their correctness. These  algorithms do not require reclassification of an ontology in any test after a first single classification before executing one or more TDD tests, and are such that the algorithms are compliant with any OWL 2 compliant reasoner. This is feasible through bypassing the ontology editor functionality (`pressing the reasoner button') and availing directly of a set of methods available from the OWL reasoners in a carefully orchestrated way. We have implemented the algorithms by extending one of the three TDD tools for ontologies, TDDOnto \cite{KL16}, into TDDonto2---also a Prot\'eg\'e 5 plugin---as a proof-of-concept to ascertain their correct functioning practically \cite{DKL17}\footnote{This open source plugin is accessible at \url{https://github.com/kierendavies/tddonto2}}. 
This implementation was subsequently used in two evaluations. First, we devised a human-independent editing efficiency approach examining clicks, keystrokes, and reasoner invocation to compare the test-last with the basic Prot\'eg\'e 5 interface to the test-first with TDDonto2. TDDonto2 has a higher editing efficiency, i.e., takes less time, than the basic interface with a small difference for very small ontologies and substantially for medium to large ontologies. Second, we conducted a typical user evaluation to compare the basic Prot\'eg\'e interface to TDDonto2, which demonstrated that the modellers completed a larger part of the tasks with fewer mistakes in less time. Thus, TDD with the test-first approach and TDDonto2  is more effective than the common (test-last with Prot\'eg\'e) authoring approach.

In the remainder of the paper, we first provide motivations for why testing is applicable to ontologies (Section \ref{sec:rationale}), and describe the main requirements. Section \ref{sec:relworks} describes related work. The first part of the main contributions are presented in Section \ref{sec:foundation}, which is the  model for testing and the main novel algorithms. 
The evaluations of its implementation are presented in Section~\ref{sec:eval}. Section \ref{sec:disc} discusses the work and we conclude in Section \ref{sec:concl}.

\section{Motivations and requirements for a test-first approach}
\label{sec:scenario} \label{sec:rationale}

Test-Driven Development (TDD) in software engineering \cite{Beck04} is a methodology based on 
two rules:
  1) write new code only if an automated test has failed, and
  2) eliminate duplication.
This induces 
a ``red--green--refactor'' pattern of development: first write a new test which fails, then write code which makes it pass with minimal effort, then remove resultant duplication and restructure if necessary.
Tests thus serve to define desired functionality.
The process is usually facilitated with a test harness that runs tests automatically and generates reports.
TDD has been shown to improve code quality \cite{Rafique:TDD}, especially in complex projects, and it is also shown to improve productivity \cite{Janzen06}.
In light of this, TDD may also be used for ontology development. In the next two subsections we first motivate where in the ontology authoring process tests could be useful and subsequently look at  the broader picture of the ontology development processes.

\subsection{TDD tests for ontologies}

Ontologies, like computer programs, can become complex so that it is difficult for a human author to predict the consequences of changes. 
Automated tests are therefore useful to detect unintended consequences.
As an illustrative example, suppose an author creates the following classes and subsumptions:
$\oent{Giraffe} \sqsubseteq \oent{Herbivore} \sqsubseteq \oent{Mammal} \sqsubseteq \oent{Animal}$. 
The author then realises that not all herbivores are mammals, so shortens the hierarchy to $\oent{Herbivore \sqsubseteq Animal}$, thereby losing the $\oent{Giraffe} \sqsubseteq \oent{Mammal}$ derivation. An application that uses this ontology to retrieve mammals would then erroneously exclude giraffes.
This issue can be caught by a simple automated test to check whether $\oent{Giraffe \sqsubseteq Mammal}$ is still entailed.
Superficially, it may seem like this problem can be solved by just adding those axioms directly to the ontology.
However, adding such axioms introduces a lot of redundancy, making modification of the ontology more difficult. 
Adding only a test instead ensures correctness without bloating the ontology. In addition, if it is specified as a test and documented as such in a testing environment, one can easily repeatedly re-run tests. 

Tests may also be used outside an automated test suite in order  
to explore and understand an ontology.
For example, an author might be assessing an ontology of animals for reuse and wants to verify that $\oent{Giraffe} \sqsubseteq \oent{Mammal}$ is entailed in the ontology.
The author can simply create a corresponding temporary test and observe the result, saving the time it would take to browse the inferred class hierarchy in an ontology development environment such as Prot\'eg\'e.

A similar approach can be employed when developing a new ontology: create a temporary test to determine whether the axiom i) is already entailed, ii) would result in a contradiction or unsatisfiable class if it were to be added to the ontology, or iii) can be added safely. For instance, one may wish to check whether one can add the domain declaration $\exists \oent{eats}^-.\top \sqsubseteq \oent{Animal}$ to the ontology such that it would not cause itself or another class to become unsatisfiable.  
In this case, it would be just executing this one TDD test for this axiom. Compare that with laborious and time-consuming alternatives: 1) browsing to the object property declarations and/or clicking through all class axioms of classes that are not animals and manually inspect each one on whether they happen to participate in an axiom involving $\oent{eats}$, or 2) the standard approach of adding an axiom, running the reasoner, and then observing the consequences. 
Option 1 may be feasible for a small toy ontology consisting of a few classes, but not for typical ontologies, let alone large ontologies, and one probably still adds option 2 to it. Option 2 involves reclassification, which may be slow, and which a TDD test can avoid once the ontology is classified. 

This gives us two broad use cases:
\begin{enumerate}
  \item
  \label{enum:usecase:regression}
  Declare many tests alongside an ontology and evaluate them all together in order to demonstrate quality or detect regressions.

  \item
  \label{enum:usecase:temp}
  Evaluate temporary tests as needed in order to explore an ontology or predict the consequences of adding a new axiom.
\end{enumerate}
To satisfy both of these, a hard requirement is that tests must not reclassify the ontology, and they must produce results that identify the consequences of adding an axiom.

\subsection{Tests within the development process}

The previous examples assume there is an axiom to add, but that has to come from some place and more activities are going on when testing. We identified three scenarios for the former and illustrated an informal lifecycle process for the latter in \cite{KL16}. For instance, i) the axiom may be a formalisation of a competency question formulated by domain experts, ii) a modeller may work with a template for axioms, a spreadsheet or an ontology design pattern for one or more axioms, or iii) a knowledge engineer may know already which axiom to add. Resolving this aspect is beyond the current scope, as even within the TDD part, multiple options already are possible even with the straightforward fail-pass character. This is illustrated in Figure~\ref{fig:tddontoIdea} for a more detailed version of the TDD test cycle within the larger TDD lifecycle for the simplest base case. TDD tests for ontologies currently only pass or fail, but, cf. TDD tests in software development, one has to deal with classifying the ontology and handling any inconsistency or unsatisfiable classes as well. The reason that that classification step is present is because when a test evaluates to a `fail', one does not know whether that axiom is absent just because the knowledge is missing or because it would result to an inconsistency if it were to be added, so one would have to check that anyway after all. Ideally, one would want to know {\em before} any edit {\em why} the axiom is not in the ontology, so as to be better informed about the next step(s) to take.

\begin{figure}[t]
\includegraphics[width=0.48\textwidth]{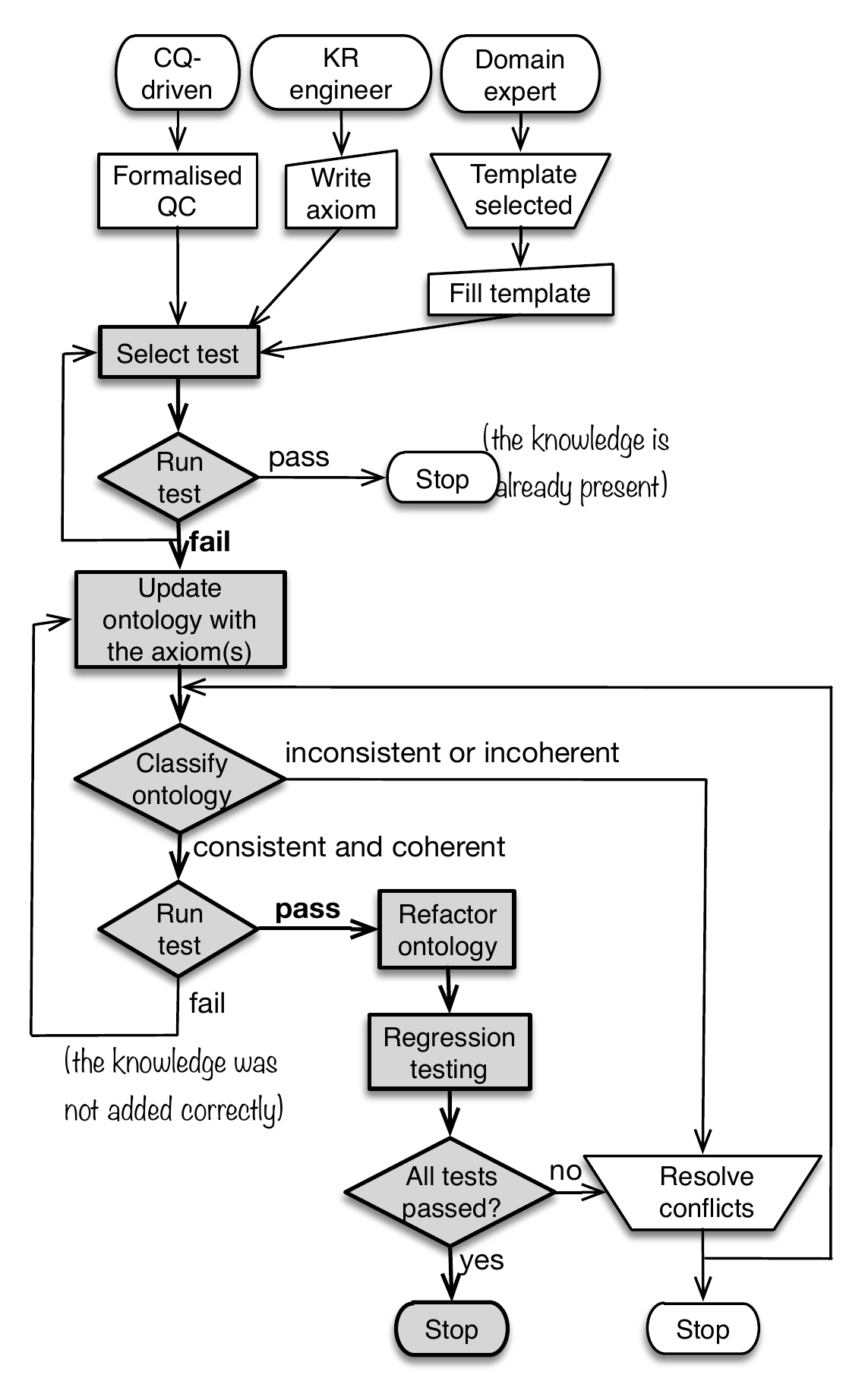}
\caption{General flow for the simplest base case for a single TDD test that corresponds to the red-green (fail-pass) TDD notion, based on the lifecycle in \cite{KL16} (not all possible permutations are shown). The highlighted part indicates the flow of steps for the straightforward fail-pass case.}\label{fig:tddontoIdea}
\end{figure}

\section{Related works}
\label{sec:relworks}

We briefly outline the few ontology testing implementations first and subsequently recap relevant aspects of automated reasoners for OWL ontologies.

\paragraph{Test-Driven Development for ontologies} 
Three TDD tools for ontologies have been proposed recently. 
TDDOnto is a Prot\'eg\'e plugin that allows test axioms to be specified in Prot\'eg\'e's syntax and then
uses the reasoner through the OWL API, 
for it was shown to be the most efficient technology examined \cite{KL16}. Further, once the ontology is classified, one can run as many implemented tests as one wants without invoking a classification again, whereas the other tools require ontology classification with each axiom exploration; hence, by design, they are already less efficient than TDDOnto.
The two other TDD approaches for ontology development use a subset of OWL or have a different scope. 
Tawny-OWL \cite{Warrender15} is an ontology development framework 
implemented in Clojure (not a widely known and used programming language) and it
provides predicate functions which query the reasoner that will return true/false. It can be used in conjunction with any testing framework, such as the built-in ``clojure.test''.
TDD from a domain expert perspective is explored with \textsc{Scone} \cite{Scone:Bitbucket}, which is based on Cucumber \cite{Cucumber}. It leverages controlled natural language and (computationally costly) mock individuals, as do the tests described in \cite{Denaux12}. 
Like Tawny-OWL, {\sc Scone} is also separated from the common ontology development tools.  Tawny-OWL and \textsc{Scone} do not support testing object properties or data properties.

No attempt has been made to rigorously prove the correctness of the testing algorithms of TDDOnto, Tawny-OWL, or \textsc{Scone}. 
None of these three tools support the full range of axioms permitted in OWL 2. Most notably, in none of them it is possible to directly test axioms of the form $C \sqsubseteq D$ where $C$ is {\em not} a named class, such as $\forall R.E \sqsubseteq D$, i.e., where it is a class expression.  
In addition, all three tools give only limited information about the result of any test, being pass/fail in 
Tawny-OWL and \textsc{Scone}, with TDDOnto also reporting missing vocabulary.
This hinders their usefulness as a means to explore an ontology or aid in development. Using TDD with arbitrary General Concept Inclusions (GCIs) and a more comprehensive testing model is possible with TDDonto2 that has been introduced recently \cite{DKL17}, but the paper did not cover technical details and evaluation.

There are several related works that move in the direction of Agile ontology development \cite{Blomqvist12,Peroni17}, such as the 
pay-as-you-go approach for Ontology-Based Data Access \cite{Sequeda17}
and the axiom-based possible world explorer 
\cite{Ferre16},  
a proposal for an ontology testing framework for requirements verification and validation 
\cite{Fernandez17}, 
and the `preparation' step of going from competency questions to computing TDD tests \cite{Dennis17}. 
They have an affinity with TDD insofar as that they could integrate with, or even rely on, a well-functioning, reliable, implementation of TDD for ontology authoring, and they predominantly rely on SPARQL queries rather than an OWL reasoner.

\paragraph{Automated reasoners}
\label{sec:reasoners}

There are numerous OWL 2-compliant automated reasoners  
that employ several mechanisms to handle OWL files; 
e.g., the OWL API \cite{OWLAPI}, OWLlink \cite{Liebig11} (a Java library that defines a widely-supported standard interface for reasoners), and   
OWL-BGP \cite{Kollia11} (a Java library that implements SPARQL). 
OWL-BGP introduces an efficiency overhead which is not present in OWL API \cite{KL16} and OWLlink specifies a protocol for communication between distributed components (availing of the OWL API) whose scenario is thus orthogonal. Therefore we consider only using the reasoner directly through the OWL API and its functionality.

Performance evaluation of TDDOnto found that implementations that introduced temporary ``mock'' individuals were substantially slower than all others \cite{KL16}.  The cause was not explicitly identified, but it is likely due to the need for reclassification of the ontology to include the new assertions.  As stipulated in Section~\ref{sec:intro}, reclassification is undesirable and therefore that approach is not appropriate when it can be avoided.  Instead, after the ontology is initially classified, one should be able to test by making queries, which are assumed to be acceptably efficient.
The OWL API reasoner interface specifies a ``convenience'' method named \textproc{isEntailed} that accepts any axiom and returns a Boolean indicating whether or not that axiom is entailed.  However, it is not mandatory for reasoners to implement this method, and only at most half do so\footnote{There are 73 reasoners listed at \url{http://owl.cs.manchester.ac.uk/tools/list-of-reasoners/} of which 38 showed evidence of maintenance since 2012, and of those, 19 reasoners list support for entailment.}.  We therefore do not want to rely on its use. 
Therefore, we will use other methods available.

\section{Foundations of TDD for ontologies}
\label{sec:foundation}

After a few preliminaries, we introduce the model for testing ontologies and subsequently present a selection of the algorithms.

\subsection{Preliminaries}

We begin with two prerequisite definitions, and then identify relevant reasoner methods and their returned values.

\begin{definition}[Ontology language $\mathcal{L}$]
$\mathcal{L}$ is the language of ontology $O$ with language specification adhering to the OWL 2 standard \cite{OWL2rec}, which has classes denoted $C, D \in V_C$ where $C$ or $D$ may be a named class or a class expression, object properties $OP \in V_{OP}$, individuals $I \in V_I$, and axioms $A \in O_{A}$ that adhere to those permitted in 
OWL 2 DL.
\end{definition}

\begin{definition}[Signatures]\label{def:sig}
  The \emph{signature} $\signature(A)$ of an axiom $A \in O_{A}$ in ontology $O$ represented in $\mathcal{L}$, is the set of all symbols in  
  the axiom. Further, there is a \emph{class signature} $\signatureC(A) \subseteq \signature(A)$, an \emph{object property signature} $\signatureOP(A) \subseteq \signature(A)$ and an \emph{individual signature} $\signatureI(A) \subseteq \signature(A)$.
\end{definition}

\begin{definition}[OWL reasoner methods]
Let $O$ be the ontology under test, 
$\{C, ...\} \in V_C$ are class expressions; $N \in V_C$ is a named class; $a, b \in V_I$; and $R \in V_{OP}$. 
  The following methods are available from the reasoner: \\
$
  \begin{array}{rcl}
   && \Call{isSatisfiable}{C} \iff O \nsementails C \sqsubseteq \bot \\
    && \Call{getSubClasses}{C} = \{ N \in \signatureC(O) \mid O \sementails N \sqsubseteq C \} \\
    && \Call{getInstances}{C} = \{ a \in \signatureI(O) \mid O \sementails a : C \} \\
  &&  \Call{getTypes}{a} = \{ N \in \signatureC(O) \mid O \sementails a : N \} \\
    && \Call{getSameIndividuals}{a} =  \{ b \in \signatureI(O) \mid O \sementails b \equiv a \} \\
 &&   \Call{getDifferentIndividuals}{a} =  \{ b \in \signatureI(O) \mid O \sementails b \not\equiv a \}
  \end{array}
  $
\end{definition}
\noindent  Note that \textproc{getSubClasses} returns the union of equivalent classes and strict subclasses.

\subsection{A model of testing ontologies}
\label{sec:model}

In order to rigorously examine any testing algorithms, we need a formal description of what it means to test an axiom against an ontology
\footnote{Observe that the scope is testing axioms, not a broad informal meaning of `testing' that also would include, say, checking for  naming conventions.}.
In line with the use cases identified in Section~\ref{sec:rationale}, we define the possible test results. 
Instead of the underspecified three possible statuses of the existing tools (pass/fail/unknown), we specify seven cases that include the three existing ones and, principally, refine the `fail' cases. They are listed in order from most grave failure to pass.
\begin{itemize}
  \item {\em Ontology already inconsistent.}  
  That is, $O \vdash \top \sqsubseteq \bot$. 
  The reasoner cannot meaningfully respond to queries, so no claims can be made about the axiom.
  \item {\em Ontology already incoherent} 
  There is at least one $C \in V_C$ such that $C \sqsubseteq \bot$. 
  \item {\em Missing entity in axiom.} 
 By Definition~\ref{def:sig}, we have $\Sigma(A) \nsubseteq \Sigma(O)$.  
  \item {\em Axiom causes inconsistency.}  If the axiom were to be added to $O$, then it would cause it to become inconsistent, i.e., $O \cup A \vdash \bot$.
  \item {\em Axiom causes incoherence.}  If the axiom were to be added to $O$, it would cause at least one named class to become unsatisfiable.
  \item {\em Axiom absent.}  The axiom is not entailed by the ontology ($O \nvdash A$) and can be added without negative consequences.
  \item {\em Axiom entailed.}  The axiom is already entailed by the ontology ($O \vdash A$). 
\end{itemize}
In the context of TDD, 
only 
``Axiom entailed'' is 
a pass; all the others are test failures. 
The first two possible failures apply to the entire suite of tests rather than to any one, so they should be checked only once as preconditions before evaluating any tests. Therefore, we do not consider them 
in any of the algorithms in Section~\ref{sec:algorithms}. 
Similarly, the missing entities case can be a simple check at the start of each test which does not affect how it is otherwise evaluated. 
Since there is no ambiguity, we henceforth abbreviate the remaining cases to ``inconsistent'', ``incoherent'', ``absent'', and ``entailed''. 
This leads to the following formal definition of the testing model.

\begin{definition}[Model for testing]\label{def:testingmodel}
  Given a consistent and coherent ontology $O$, and an axiom $A$ s.t. $\signature(A) \subseteq \signature(O)$, i.e., 
    \[
    \mbox{\em pre-test}_O(A) =
    \begin{cases}
      O \nvdash \top \sqsubseteq \bot &
        \text{ontology is consistent}  \\
        C \in V_C \text{ s.t. } C \sqsubseteq \neg \bot &
        \text{ontology is coherent} \\
      \Sigma(A) \subseteq \Sigma(O) &
        \text{axiom vocabulary} \\
        & \text{\hspace{2mm} elements present in } O \\
    \end{cases}
  \]
  then the result of testing $A$ against $O$ is:
  \[
    \test_O(A) =
    \begin{cases}
      \entailed &
        \text{if } O \sementails A \\
      \inconsistent &
        \text{if } O \union A \sementails \bot \\
      \incoherent &
        \text{if }
        \! \begin{aligned}[t]
          & O \union A \nsementails \bot  \land (\exists C \in \signatureC(O)) \land \\
          & O \union A \sementails C \sqsubseteq \bot
        \end{aligned} \\
      \absent &
        \text {otherwise} \\
    \end{cases}
  \]
\end{definition}
  The resultant values are ordered according to graveness of failure: $\entailed < \absent < \incoherent < \inconsistent$.
One could add nicer labels in a user interface, such as ``this axiom is redundant'' for ``entailed'', but such considerations are outside the current scope.

\subsection{Algorithms and analysis}
\label{sec:algorithms}

We now introduce the algorithms and analyses, in the context of an ontology $O$, which cover the most used types of axioms for (complex) classes, individuals, and RBox axioms that can be expressed as class axioms (domain and range axioms, and functional and local reflexivity and their inverses)\footnote{currently not supported: 1) entity declarations and datatype definitions because they cannot meaningfully be tested, 2) $\oaxiom{HasKey}$ axioms, for they are hardly used due to unusual semantics, 3) RBox axiom types other than listed above, because it is complex to detect inconsistencies \cite{Keet12ekaw} and they require non-standard reasoning services.}. 
In the interest of space and readability, a selection of the algorithms and proofs is presented in this section, which is based on importance, novelty, and to demonstrate the approach to the algorithms and proofs. The remaining ones follow the same pattern and are explained briefly in the text; the complete set of algorithms and their proofs can be found in a technical report \cite{Davies16}.

Overall, this leaves four generic class axioms, three assertions, and six object property axioms. 
Each algorithm is named according to the axiom it tests, as written in OWL 2 functional syntax, prepended with ``\textproc{test}''.  For example, the algorithm for testing $\oaxiom{SubClassOf}$ axioms is named \textproc{testSubClassOf}.

\subsubsection{Class axioms}
\label{sec:algorithms:class}

In the class axioms permitted by OWL 2 DL, all arguments may be arbitrary class expressions, not just named classes, except for $\oaxiom{DisjointUnion}(N, C_1, \ldots, C_n)$ in which $N$ must be a named class.  Consequently, to determine if $C \sqsubseteq D$ holds, it is not sufficient to check if $C \in \Call{getSubClasses}{D}$, because $C$ will not occur in this set if it is not a named class.  To resolve this, we build class expressions from the arguments and query them for satisfiability and instances.

To test such GCIs, we introduce  Algorithm~\ref{alg:testSubClassOf}, which tests subsumption of class expressions and we will show its correctness. 

\begin{algorithm}[H]
  \caption{test $C \sqsubseteq D$}
  \label{alg:testSubClassOf}
  \begin{algorithmic}[1]
    \raggedright
    \Input{$C, D$ class expressions}
    \Function{testSubClassOf}{$C, D$}
      \If{$\Call{getInstances}{C \sqcap \neg D} \ne \emptyset$}
        \label{alg:testSubClassOf:checkInconsistent}
        \State \Return $\inconsistent$
        \label{alg:testSubClassOf:returnInconsistent}
      \ElsIf{$\Call{getSubClasses}{C \sqcap \neg D} \ne \emptyset$}
        \label{alg:testSubClassOf:checkIncoherent}
        \State \Return $\incoherent$
        \label{alg:testSubClassOf:returnIncoherent}
      \ElsIf{\Call{isSatisfiable}{$C \sqcap \neg D$}}
        \label{alg:testSubClassOf:checkAbsent}
        \State \Return $\absent$
      \Else
        \State \Return $\entailed$
        \label{alg:testSubClassOf:returnEntailed}
      \EndIf
    \EndFunction
  \end{algorithmic}
\end{algorithm}


\begin{lemma}
  \label{lem:subclassEntailsUnsat}
  For any set of axioms $O$ and class expressions $C$ and $D$,
  $O \:\sementails\: C \sqsubseteq D \;\iff\; O \:\sementails\: C \sqcap \neg D \sqsubseteq \bot $.
\end{lemma}
This is well known and therefore the proof is not included.

\begin{proposition}
  \label{prop:testSubClassOfEntailedSound}
  \label{prop:testSubClassOfEntailedComplete}
  \textproc{testSubClassOf} is sound and complete for 
  entailment.  That is,
$    \Call{testSubClassOf}{C, D} = \entailed \implies 
    \test_O(C \sqsubseteq D) = \entailed$
 and 
    $ \test_O(C \sqsubseteq D) = \entailed \implies 
    \Call{testSubClassOf}{C, D} = \entailed
    $.
\end{proposition}
\begin{proof}
\underline{Soundness}
  Algorithm~\ref{alg:testSubClassOf} can only return entailed at line \ref{alg:testSubClassOf:returnEntailed}, so the three {\bf if}-conditions must all be false.  So
  \vspace{-1mm}
  \begin{align}
    \Call{getInstances}{C \sqcap \neg D} = \emptyset\,\, \land 
    {} & \Call{getSubClasses}{C \sqcap \neg D} = \emptyset\,\, \land \nonumber \\
     {} & \lnot \Call{isSatisfiable}{C \sqcap \neg D} \label{eqn:testSubClassOfEntailed} 
    \end{align}
  \vspace{-1mm}
  Now suppose $O \nsementails C \sqsubseteq D$.  By Lemma~\ref{lem:subclassEntailsUnsat}, $O \nsementails C \sqcap \neg D \sqsubseteq \bot$.  In other words, $C \sqcap \neg D$ is satisfiable, which contradicts the last term of Eq.~\ref{eqn:testSubClassOfEntailed}.  Hence the supposition is false, so
    $ O \sementails C \sqsubseteq D 
    \iff{} 
    \test_O(C \sqsubseteq D) = \entailed
    $.

\underline{Completeness}
  The algorithm returns entailed if Eq.~\ref{eqn:testSubClassOfEntailed} holds (see Soundness).
  From $\test_O(C \sqsubseteq D) = \entailed$ we have that $O \sementails C \sqsubseteq D$, and by Lemma~\ref{lem:subclassEntailsUnsat}, $O \sementails C \sqcap \neg D \sqsubseteq \bot$ so the last term of the equation is true.
  Since $C \sqcap \neg D$ is unsatisfiable, by the coherence precondition it has no named subclasses, and by the consistency precondition it has no instances.  Therefore the first and second terms of the equation are also true.
  Therefore Eq.~\ref{eqn:testSubClassOfEntailed} holds, and so the algorithm returns entailed. 
\end{proof}

\begin{proposition}
  \label{prop:testSubClassOfInconsistentSound}\label{prop:testSubClassOfInconsistentComplete}
  \textproc{testSubClassOf} is sound and complete w.r.t. inconsistency.
\end{proposition}
\begin{proof}
\underline{Soundness}
  Algorithm~\ref{alg:testSubClassOf} can only return inconsistent at line \ref{alg:testSubClassOf:returnInconsistent}, so the first {\bf if}-condition holds, so
  $\Call{getInstances}{C \sqcap \neg D} \ne \emptyset $, 
  which means there exists an individual $a$ such that
    $a : C \sqcap \neg D 
    \iff{} 
     a : C \;\land\; a : \neg D $. 
%
  Under $O \union (C \sqsubseteq D)$ it follows also that $a : D$, which is a contradiction, so
    $ O \union (C \sqsubseteq D) \sementails \bot 
    \iff{}
    \test_O(C \sqsubseteq D) = \inconsistent$. 

\underline{Completeness} 
  We have that $O$ is consistent, so it has an interpretation, but $O \union (C \sqsubseteq D)$ is inconsistent, so it has no interpretations.
  Suppose the algorithm does not return inconsistent.  Then it must be that
  $\Call{getInstances}{C \sqcap \neg D} = \emptyset$. 
  But in this case there exists an interpretation $\I$ which models both $O$ and $O \union (C \sqsubseteq D)$.  Let $\I$ be the interpretation of $O$ with the smallest domain.  This means that the interpretation of any class $E$ only contains elements which correspond to individuals which must be in that class: 
  $E^\I = \{ x \in \Delta^\I \mid (\exists a \in \signatureI(O)) \; a : E \land a^\I = x \}$. 
  This clearly still models $O$ because every individual is still in all classes it is entailed to be in.
  Under the supposition, we have that
  $(C \sqcap \neg D)^\I = \emptyset$. 
  So for any individual $a$,
  $a : \neg (C \sqcap \neg D)$. 
  Letting $a : C$, then 
    $a : \neg (C \sqcap \neg D) 
    \implies{} 
    a : \neg C \sqcup D 
    \implies{} 
    a : D $.
  From the construction of $\I$, this means that
    $C^\I \subseteq D^\I 
    \implies{} 
    \I \synentails C \sqsubseteq D$. 
  So $\I$ also models $O \union (C \sqsubseteq D)$.
  This contradicts the initial condition that $O \union (C \sqsubseteq D)$ is inconsistent, so the supposition must be false, and therefore the algorithm returns inconsistent. 
\end{proof}

\begin{proposition}
  \label{prop:testSubClassOfIncoherentSound}
    \label{prop:testSubClassOfIncoherentComplete}
  \textproc{testSubClassOf} is sound and complete w.r.t. incoherence.
\end{proposition}
\begin{proof}
\underline{Soundness}
  Algorithm~\ref{alg:testSubClassOf} can only return incoherent at line \ref{alg:testSubClassOf:returnIncoherent}, so the first {\bf if}-condition must be false and the second true.  So
    $\Call{getInstances}{C \sqcap \neg D} = \emptyset 
    {} \land \Call{getSubClasses}{C \sqcap \neg D} \ne \emptyset $
  Therefore, by the second term, there exists some named class $N \in \signatureC(O)$ such that
  $N \sqsubseteq C \sqcap \neg D$.  
  By the contrapositive of Proposition~\ref{prop:testSubClassOfInconsistentComplete}, $O \union (C \sqsubseteq D)$ is consistent, so by Lemma~\ref{lem:subclassEntailsUnsat},
  \begin{align}
    & O \union (C \sqsubseteq D) \:\sementails\: C \sqcap \neg D \sqsubseteq \bot 
    \implies{} O \union (C \sqsubseteq D) \:\sementails\: N \sqsubseteq \bot  \nonumber \\
    & \implies{} \test_O(C \sqsubseteq D) = \incoherent \nonumber
  \end{align}

\underline{Completeness}
  If 
  $C \sqsubseteq D$ is added to $O$, then the only classes that are affected are $C$ and its subclasses.  Consider a named class $N \sqsubseteq C$.  If $O \nsementails N \sqsubseteq \neg D$ then it is possible that any element in the $N^\I$ is also in $D^\I$ and thus it is possible that $N \sqsubseteq D$.  If this is true for all such $N$, then they are all satisfiable in $O \union (C \sqsubseteq D)$ which is therefore coherent, so it must not be true for at least one $N$.  That is,
  \begin{align}
    & (\exists N \in \signatureC(O)) \;\; N \sqsubseteq C \land N \sqsubseteq \neg D \nonumber \\
    & \iff{} (\exists N \in \signatureC(O)) \;\; N \sqsubseteq C \sqcap \neg D \nonumber \\
    & \iff{} (\exists N \in \signatureC(O)) \;\; N \in \Call{getSubClasses}{C \sqcap \neg D} \nonumber \\
    & \iff{} \Call{getSubClasses}{C \sqcap \neg D} \ne \emptyset  \nonumber
  \end{align}
  From the contrapositive of Proposition~\ref{prop:testSubClassOfInconsistentComplete} we have that the first {\bf if}-condition is false, and we have shown that the second {\bf if}-condition is true, so the algorithm returns incoherent. 
\end{proof}

\begin{theorem}
  \label{thm:testSubClassOf}
  \textproc{testSubClassOf} is correct and terminating.
\end{theorem}
\begin{proof}
  It has been shown that Algorithm~\ref{alg:testSubClassOf} is sound and complete for entailment, inconsistency, and incoherence, and the result is absent when it is not one of these other three.  Therefore the algorithm returns the correct result in all cases. 
  Termination is trivial, since the algorithm contains no loops or recursion. 
\end{proof}

Note that a test for local reflexivity can avail of the same algorithm, with  \Call{testSubClassOf}{$\top$,ObjectHasSelf(R)} and likewise for irreflexivity, and it holds similarly for functional as  \Call{testSubClassOf}{$\top$,ObjectMaxCardinality(1,R)}, and likewise for inverse functional, and object property domain and range (see \cite{Davies16} for details).


%
Testing for equivalent classes is done with a \textproc{testEquivalentClasses} function. The algorithm is correct and terminating, which follows directly from 
its specification: it iterates through \textproc{testSubClassOf} in a nested {\bf for}-loop for classes $C_1, ..., C_n$ (with $n \geq 2$), and given that \textproc{testSubClassOf} is correct and terminating, then \textproc{testEquivalentClasses} so is the former
 (see \cite{Davies16} for details).

Algorithm \ref{alg:testDisjointClasses} for disjoint classes
is sound and complete for entailment, inconsistency, and  incoherence, following largely the proofs of \textproc{testSubClassOf} that it uses within its {\bf for}-loop, and therefore not included here 
(These proofs are available in the online technical report \cite{Davies16}). Likewise,  Algorithm \ref{alg:testDisjointUnion} for disjoint union is correct and terminating, relying on the previous results for the test for equivalent classes and for disjoint classes.

\begin{algorithm}[H]
  \caption{ test $\oaxiom{DisjointClasses}(C_1, \ldots, C_n)$}
  \label{alg:testDisjointClasses}
  \begin{algorithmic}[1]
    \raggedright
    \Input{
      $C_1, \ldots, C_n$ class expressions \\
      $n \ge 2$
    }
    \Function{testDisjointClasses}{$C_1, \ldots, C_n$}
      \State $r \gets \entailed$
      \For{$i \gets 1$ \To $n-1$}
        \For{$j \gets i+1$ \To $n$}
          \State $r' \gets \Call{testSubClassOf}{C_i, \neg C_j}$
          \label{alg:testDisjointClasses:inner}
          \State $r \gets \max(r, r')$
        \EndFor
      \EndFor
      \State \Return $r$
    \EndFunction
  \end{algorithmic}
\end{algorithm}

\begin{algorithm}[H]
  \caption{ test $\oaxiom{DisjointUnion}(N, C_1, \ldots, C_n)$}
  \label{alg:testDisjointUnion}
  \begin{algorithmic}[1]
    \raggedright
    \Input{
      $N$ named class \\
      $C_1, \ldots, C_n$ class expressions \\
      $n \ge 2$
    }
    \Function{testDisjointUnion}{$N, C_1, \ldots, C_n$}
      \State $r_1 \gets \Call{testEquivalentClasses}{N, C_1 \sqcup \ldots \sqcup C_n}$
      \State $r_2 \gets \Call{testDisjointClasses}{C_1, \ldots, C_n}$
      \State \Return $\max(r_1, r_2)$
    \EndFunction
  \end{algorithmic}
\end{algorithm}

\subsubsection{Assertions}
\label{sec:algorithms:assert}

Observe that adding an assertion does not affect satisfiability of classes, so $O$ cannot become incoherent.  We take this as given for all axioms tested in this section.
Algorithm \ref{alg:testSameIndividual} tests equivalence and Algorithm \ref{alg:testDifferentIndividuals} difference of individuals. 
When an algorithm accepts $n$ individuals $\{a_1, \ldots, a_n\}$ as arguments, we use the shorthand $\mathbf{a}$ for this set. 
  We use the integer variable $i$ to iterate over the indices of individuals given as arguments, and the variables $\delta$ and $\gamma$ to temporarily store a set of individuals. Soundness and completeness proofs are available in the online technical report \cite{Davies16}.
  
%

\begin{algorithm}[h]
  \caption{test $a_1 \equiv \ldots \equiv a_n$}
  \label{alg:testSameIndividual}
  \begin{algorithmic}[1]
    \raggedright
    \Input{
      $a_1, \ldots, a_n$ individuals \\
      $n \ge 2$
    }
    \Function{testSameIndividual}{$a_1, \ldots, a_n$}
      \If{$\{a_2, \ldots, a_n\} \subseteq \Call{getSameIndividuals}{a_1}$}
        \State \Return $\entailed$
        \label{alg:testSameIndividual:returnEntailed}
      \Else
        \For{$i \gets 1$ \To $n$}
          \State $\delta \gets \Call{getDifferentIndividuals}{a_i}$
          \If{$\{a_1, \ldots, a_n\} \intersect \delta \ne \emptyset$}
            \State \Return $\inconsistent$
            \label{alg:testSameIndividual:returnInconsistent}
          \EndIf
        \EndFor
        \State \Return $\absent$
      \EndIf
    \EndFunction
  \end{algorithmic}
\end{algorithm}
\vspace{3mm}

\begin{algorithm}[h]
  \caption{ test $\oaxiom{DifferentIndividuals}(a_1, \ldots, a_n)$}
  \label{alg:testDifferentIndividuals}
  \begin{algorithmic}[1]
    \raggedright
    \Input{
      $a_1, \ldots, a_n$ individuals \\
      $n \ge 2$
    }
    \Function{testDifferentIndividuals}{$a_1, \ldots, a_n$}
      \For{$i \gets 1$ \To $n$}
        \State $\gamma \gets \Call{getSameIndividuals}{a_i}$
        \If{$(\{a_1, \ldots, a_n\} \setminus \{a_i\}) \intersect \gamma \ne \emptyset$}
          \State \Return $\inconsistent$
          \label{alg:testDifferentIndividuals:returnInconsistent}
        \EndIf
      \EndFor
      \For{$i \gets 1$ \To $n$}
        \State $\delta \gets \Call{getDifferentIndividuals}{a_i}$
        \If{$(\{a_1, \ldots, a_n\} \setminus \{a_i\}) \nsubseteq \delta$}
          \State \Return $\absent$
        \EndIf
      \EndFor
      \State \Return $\entailed$
      \label{alg:testDifferentIndividuals:returnEntailed}
    \EndFunction
  \end{algorithmic}
\end{algorithm}

Finally, the algorithm for class assertions 
checks whether the individual is an instance of a class expression, using the \Call{getInstances}{C} function. 

\section{Evaluation of TDD with TDDonto2}
\label{sec:eval}

\begin{figure*}[t]
\includegraphics[width=1.0\textwidth]{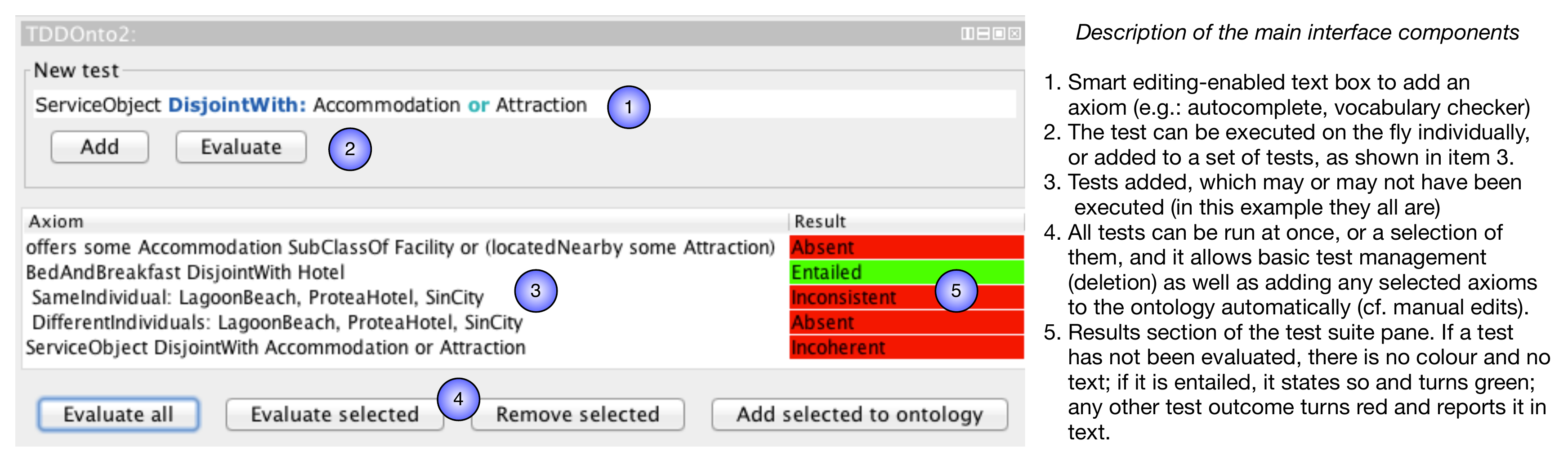}
\caption{Annotated screenshot of TDDonto2 with a few evaluated sample axioms and a tourism ontology.}\label{fig:tddonto2example}
\end{figure*}

In order to evaluate the model of testing and the algorithms, we developed a Pro\'teg\'e plugin, named TDDonto2, that implements these TDD features, which will be introduced in Section~\ref{sec:tool}. The actual evaluation will be presented afterwards. 
We carry out two ontology authoring tests: a quantitative approach to editing efficiency (Section~\ref{sec:editing}) and an experiment with 25 novice modellers (Section~\ref{sec:userstudy}). 

\subsection{The TDDonto2 Prot\'eg\'e plugin}
\label{sec:tool}

TDDonto2 \cite{DKL17} is an updated version of TDDOnto \cite{KL16} and is also a Prot\'eg\'e 5.x plugin that uses the OWL API, 
as that was shown to have resulted in the best performance cf the other techniques investigated \cite{KL16,LK16}. It implements the new model of testing for TDD and the new TDD algorithms\ and also has a GUI to wrap around it that incorporates some of the Prot\'eg\'e features, such as autocomplete and recognising whether a term is in the ontology's vocabulary. 
The plugin can be added to any tab in Prot\'eg\'e, as desired\footnote{it thus still allows the modeller to browse to the standard interface with the classes, object properties, and individuals tabs and so forth to author ontologies.}. 
An annotated screenshot of the plugin with several tests is included in Figure~\ref{fig:tddonto2example}. 
Within the broader setting of an overarching TDD methodology for ontology authoring, it covers the aspects from entering the axiom (wherever it came from) until (but excluding) the refactoring step. 
Like its predecessor TDDOnto, one can type a single axiom and evaluate it directly, add several axioms and evaluate the whole a set of axioms or a selected subset thereof, and add a single axiom or a set of axioms to the ontology with a one-click operation. The different possible statuses for the test outcomes are colour-coded, where `no evaluation' is left blank, `entailed' is highlighted in green, and the other statuses are highlighted in red.  Examples that illustrate the tool and a tutorial-style screencast are available from \url{https://
github.com/kierendavies/tddonto2}, as well as the source code and jar file.

Note that TDDonto2 is a proof-of-concept tool primarily to test the workings of the algorithms and investigate in more depth what the best components of a full TDD methodology would be, i.e., it has not been subjected to a full software development lifecycle and it does not yet cover all steps in a TDD methodology (e.g., it does not consider refactoring). 
Nonetheless, it serves to evaluate the TDD with TDDonto2 to obtain first indications whether there is any benefit to a {\em test-first} approach already.

\subsection{Editing efficiency evaluation}
\label{sec:editing}

Ontology authoring efficiency consists of two components: the number of clicks and keystrokes one has to carry out even when one is familiar with the interface 
and the repeated invocation of the reasoner with as one extreme case the reasoner invocation after each edit and the other one only at the end of all editing operations. The former depends mainly on the axioms and the interface design of the ontology editor and the latter is an orthogonal dimension that can be added based on the number of edits. 

Keystroke models are well-known in Human-Computer Interaction research  as a non-invasive way to determine how long a task will take with the software, which has been extended over the years to determine correlations between typing \& browsing speed and programming performance, the emotional state of a user, and biometric identification of users to detect hackers (see \cite{Kolakowska13,Thomas05} and references therein). 
For the editing efficiency evaluation of {\em test-first} TDD-based ontology authoring with TDDonto2 vs the {\em test-last}-based baseline, we are interested in task completion, and within that scope, to eliminate the noise of users so as to get a clear understanding of the `number of clicks and keystrokes' component and the effects of  reasoner invocation to determine the best performance in the most efficient possible situation. This then puts forward the hypothesis that:
\begin{itemize}
\item[\textbf{H1}] Given a time allocation to clicks and keystrokes and automated reasoning (classification), the overall editing time is lower for TDD in TDDonto2 than the case of the `expert user' and its test-last approach.
\end{itemize}
The `expert user' is idealised as someone who is an experienced ontology engineer, which is someone who has a high familiarity with the widely-used Prot\'eg\'e 5.x tool and who never browses to the wrong place nor gets sidetracked in the authoring process.  
That said, the evaluation method 
described in the next section easily can be amended for another ontology editor and a `novice factor' can be added\footnote{for instance, slower that average typing, more latency in scrolling and moving the mouse, and adding ``Mental Operator'' time penalties for switching contexts (including to the wrong ones, like the wrong tab), and keyboard/mouse input devices; see also \cite{Kolakowska13,Thomas05} for further possible variables and metrics.}.

\subsubsection{Materials and methods}

\paragraph{Methods}
To falsify, or validate, the hypothesis, we need a systematic and plausible scenario that can handle axiom input changes. The first aspect to establish is the type of edits. We assume a typical ontology that is not too lightweight, yet also does not use all corner cases of OWL 2 DL axioms and data about common axioms in an ontology. 
This resulted in the following set of 10 types of axioms, where $C$ and $D$ are named classes, $R$ is a simple object property, $c1$ an individual\footnote{There are obviously variations, but the main point here is the principle of how to approach computing an editing efficiency and what variables are relevant.}:
\begin{enumerate}[label=(\roman*)]  
\item simple class subsumption $C \sqsubseteq D$
\item simple existential (all-some) or simple universal $C \sqsubseteq \exists R.D$ or $C \sqsubseteq \forall R.D$
\item simple disjointness $C \sqcap D \sqsubseteq \bot$ or $C \sqsubseteq \neg D$
\item  domain $\exists R \sqsubseteq C$ 
\item range $\exists R^- \sqsubseteq D$
\item instance declaration $c1:C$
\item qualified cardinality constraint  $C \sqsubseteq\,\, \geq n\,\, R.D$ or  $C \sqsubseteq\,\, \leq n\,\, R.D$ or  $C \sqsubseteq\,\, = n\,\, R.D$
\item non-simple class on the left-hand side (lhs) (other than domain and range axiom)
\item `closure' axiom $C \sqsubseteq \exists R.D \sqcap \forall R.D$
\item arbitrary class expression on the rhs, conjunction/disjunction $C \sqsubseteq \exists R.D \sqcap \exists S.(E \sqcup F)$
\end{enumerate}
Let us assume for now that these 10 axioms types are the only ones,  
so as to show the principle of the calculations. 

The second dimension for the editing efficiency calculation is the presence/absence in the ontology of all the vocabulary elements used in the axioms, i.e., whether the terms have yet to be typed up or not. Third, whether all those vocabulary elements are top-level entities in the hierarchy (i.e., directly subsumed by {\tt owl:Thing} and {\tt topObjectProperty}) versus all entities are located at the  leaves in the deepest hierarchy in the ontology. One then can compute three core variants: {\em a)} the lower bound with existing vocabulary + top-level elements, {\em b)} the upper bound with new vocabulary + down in the hierarchy, and {\em c)} an average case by taking the average characters/term and hierarchy depth for a typical online ontology or computed for the specific ontology under evaluation. We opt for the latter option, i.e., we take a quantitative approach where pre-selected axiom types are filled with average vocabulary from an existing ontology.  
To do this, we need the following variables:
\begin{itemize}
\item $a = |character_{term}|$ where term is the class name or object property name;
\item $b = |clicks\,\, for\,\, hierarchy\,\, depth|$, be this the class or object property hierarchy;
\item $c = |character_{term}|$ where term is the instance name;
\item their respective averages, $\bar{a}$, $\bar{b}$, and $\bar{c}$, for some arbitrary ontology or for a specific ontology that is to be edited.
\end{itemize}

The number of clicks and keystrokes for a set of editing operations is then calculated as follows. Let $n$ be the  number of axioms to add, which are of a form as one of the 10 above and categorised as such. Then for each axiom $A_i$ ($1 \leq i \leq n$), one can calculate the minimum of clicks ($k_i^{min}$) and the maximum ($k_i^{max}$) so that the {\em best case scenario} is $\sum_{1}^{n} k_i^{min}$ and the  {\em worst case scenario} is $\sum_{1}^{n} k_i^{max}$. Alternatively, one can select an option for each axiom that mirrors one's habit or fix the option to systematically compare the same actions across ontologies. Because this is the first such evaluation, we choose the latter case.

Next, the way the number of clicks and keystrokes is calculated for each axiom specifically is as follows. 
We examined the editing efficiency calculations by using Prot\'eg\'e 5.2, which offers several ways of adding an axiom in most cases, under the assumption that the vocabulary is already present in the ontology. The first one for  axiom of type (i) is included here and the others are listed in Appendix~A: 
\begin{enumerate}[label=(\roman*)]  
\item 
click `Classes' tab (= 1), then either: 
\begin{enumerate}
\item drag class to position, if sufficiently nearby (existing classes) = 1
\item click class - click SubClass Of - in `class expression editor' type classname - click ok (existing classes) = $3 + a_C$ 
\item click class - click SubClass Of - in `class hierarchy' click as far down in the hierarchy as needed - select class - click ok (existing or new classes) = $4 + b_C$
\end{enumerate}
\end{enumerate}
The corresponding one for TDDonto2 is $a_C$ SubClassOf: $a_D = a_C + 11 + a_D  = 11 + a_C + a_D$.

Further, when typing the name of a vocabulary element or keyword, we consider the  autocomplete feature. For a term, this varies by the size and naming, but for keywords (e.g., {\sf\small SubClassOf:}) this is one character + a tab. To even this out, such an autocomplete is set on 4 keystrokes. For instance, the above calculation for $a_C$ SubClassOf: $a_D$ in TDDonto2 then becomes 4 + 4 + 4 = 12.

For the principal case to examine, we factor in time in two ways, 
with the first one being the determining one for falsifying hypothesis H1:
\begin{itemize} 
	\item reasoning time, by the worst case scenario for reasoning where after each axiom, the reasoner is invoked in the standard Prot\'eg\'e vs an ``evaluate all'' in TDDonto2. 
	\item allocating 1 second to each click and 0.3 second to each keystroke, which is based on the typing speed average of 190-200 characters per minute (and permutations thereof; see below); 
\end{itemize}	
For small ontologies that classify fast, the time-per-click is expected to be the major factor, whereas for lager or complex ontologies, the reasoning time is expected to be the major factor. The reasoning time is determined by classifying the ontology once, and then the worst case (reasoner invocation after each axiom) is computed by multiplying it by 9 as approximation of the total cost for Prot\'eg\'e and twice for TDDonto2 (at the start and after adding all axioms).

Lastly, the principal case will be checked against several permutations to assess robustness of results. They are as follows: (1) the same setting of 1s click and 0.3s keystroke but without autocomplete, (2) without autocomplete but slower clicking (2s) and faster typing (0.25s/keystroke), and (3) the same setting but 8 keystrokes for autocomplete cf. 4.

\paragraph{Materials}

We selected three actual ontologies to compute the average values and augmented it with three `mock' ontology averages. The actual ontologies are the African Wildlife Ontology (AWO), the Pizza ontology, and DataMining Optimization Ontology (DMOP). The AWO is a very small tutorial ontology (31 classes, 5 object properties, and 56 logical axioms, in $\mathcal{SRI}$ (i.e., OWL 2 DL)), which is used in the ontology engineering course of one of the authors\footnote{http://www.meteck.org/teaching/OEbook/ontologies/} and it also will be used for the user evaluation in the next section. The Pizza ontology \cite{Rector04} is deemed well-known; it has 98 classes, 8 object properties, and 786 logical axioms, and is in $\mathcal{SHOIN}$ (OWL DL) expressiveness. DMOP \cite{KLetal15} is medium-sized and complex (723 classes, 96 object properties, and 2425 logical axioms, in $\mathcal{SROIQ}(D)$, i.e., OWL 2 DL). It was chosen because of its size and complexity and because two of the authors were involved in its development and thus would be able to analyse some internals if the results would demand for it. In addition, we added three mock ontologies such that the parameters have different values from AWO, Pizza, and DMOP such that the effects of the variables' values for reasoner time, length of names, and hierarchy depth can be examined further. Their relevant characteristics are included in Table~\ref{tab:avgdata}. 
Given that we calculate with averages, the actual number of classes etc in M1-M3 do not matter, nor does the DL fragment or OWL species as indication of possible reasoning times, for that is also a fixed value for each mock ontology, which was set to be in-between Pizza's insignificant classification time on the one end and DMOP's 20 minutes at the other end. While there are larger ontologies with longer classification times, 20 minutes is already substantial for the authoring process and much longer than the overhead of the TDDonto2 tool to compute the outcome of the tests.

The classification times for AWO, Pizza, and DMOP were recorded on a MacBook Pro with 2.7 GHz Intel Core i5 and 16 GB memory.

\begin{table}[t]
\caption{Classification time in seconds (``Classif. (s)''), average number of characters for names ($a$) of class ($_C$) and object properties ($_{OP}$) and hierarchy depth ($b$), and instances ($c$) used in the editing efficiency computation. M1-M3: mock ontologies where values are set to assess their effects on editing efficiency.}
\label{tab:avgdata}
\begin{tabular}{|l|c|c|c|c|c|c|}
 \hline
 & AWO & Pizza & DMOP & M1 & M2 & M3\\ \hline
Classif. (s)	&	0.81	 & 0.1 &	1196.53 & 100 & 500 & 25 \\ \hline
$a_C$	&	7.06	 & 13.07 &	21.09 & 15 & 15 & 23 \\ \hline
$b_C$	&	2	&4.86	& 8.39 & 6 & 12 & 6 \\ \hline
$a_{OP}$	 &	9.4	& 11.63 &	14.14 & 12 & 12 & 15 \\ \hline
$b_{OP}$	 &	1.2 &	1.5 &	2.2 & 2 & 3 & 1.5 \\ \hline
    $c$    &	0 &	6.4 &	19.03 & 10 & 10  & 19 \\ \hline
\end{tabular}
\end{table}

\begin{figure*}[t]
\centering
\includegraphics[width=1.0\textwidth]{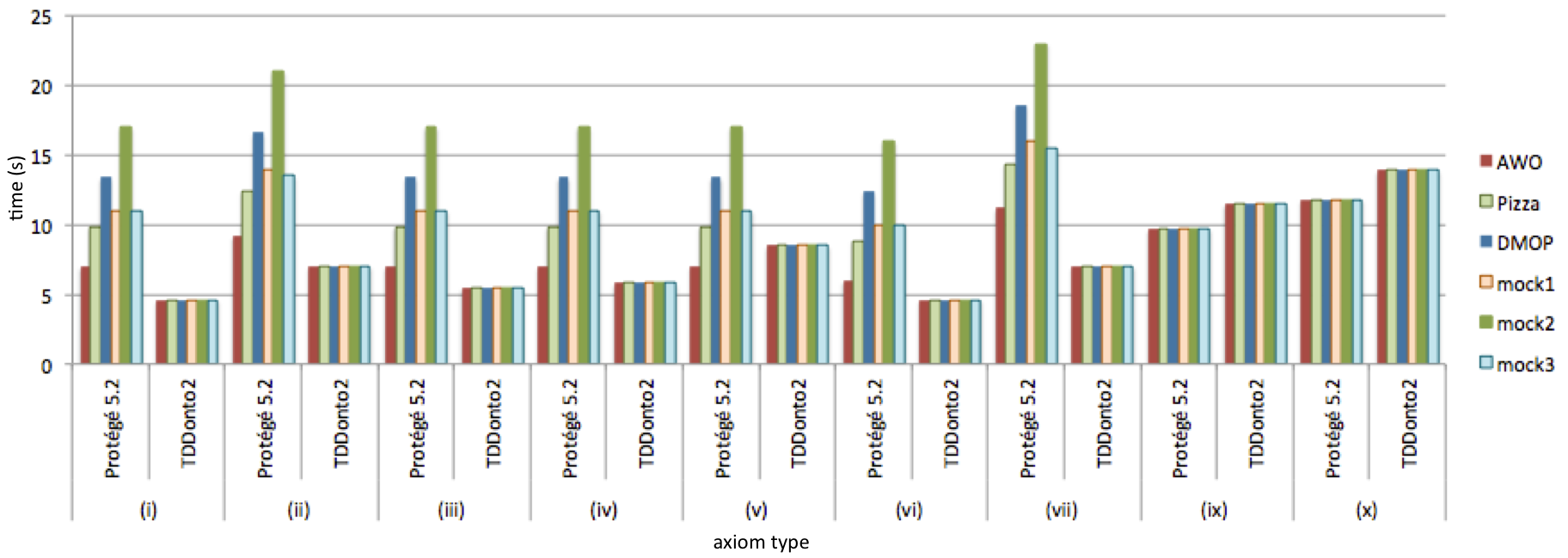}
\caption{Illustrative typical results of the editing efficiency data for Prot\'eg\'e v5.2 and TDDonto2. The times increase with increasing vocabulary name size and hierarchy depth averages (see Table~\ref{tab:avgdata}).}\label{fig:clickdetails}
\end{figure*}

\subsubsection{Results}

The aggregate results in editing efficiency time with and without reasoner is shown in Figure~\ref{fig:clickSample}. 
In order to falsify or validate hypothesis H1 on test-last vs test-first, compare ``Total Prot\'eg\'e - single edit reasoner'' with ``Total TDDonto2 - with reasoner'': it is obvious that test-first is somewhat to much faster than test-last, ranging from 13s in case of the AWO to a staggering 8430s with the DMOP. Thus, H1 is validated.

Comparing the first two data series (i.e., without reasoner) to the latter two (with reasoner), it is clear that the reasoner has most effect on overall editing time for the medium to large ontologies (DMOP and mock2). For the two small ontologies, AWO and Pizza, it turns out that the click \& keystroke time is the major contributor to the overall time taken. 
AWO's clicks amount to 75.9s in Prot\'eg\'e, so with a single classification being 0.81s, it only reaches a total of 83.19s for Prot\'eg\'e's worst case (invoke reasoner after each edit), which is 68.4s and 70.02s, respectively for TDDonto2 with its two invocations of the reasoner. 

\begin{figure}[h]
\centering
\includegraphics[width=0.48\textwidth]{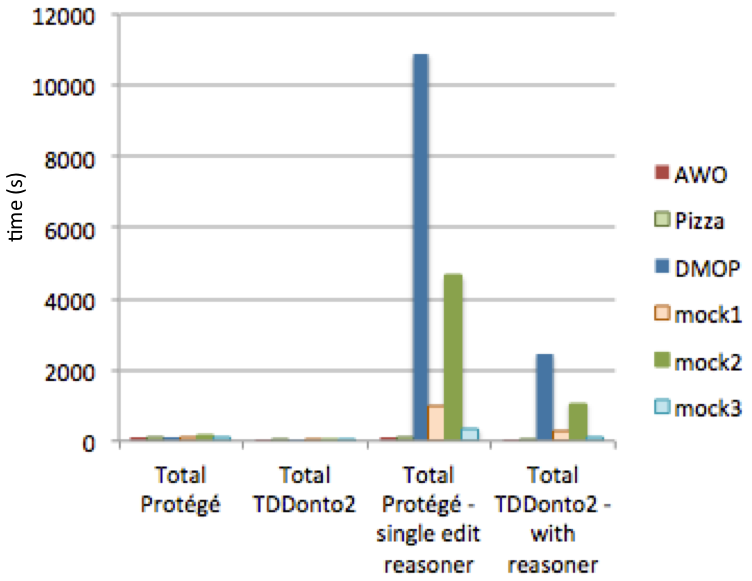}
\caption{Aggregate results for the Prot\'eg\'e v5.2 and TDDonto2 editing efficiency with and without factoring the reasoning time, when an instance of each axiom type is tested/added once.}\label{fig:clickSample}
\end{figure}

\begin{figure}[h]
\centering
\includegraphics[width=0.37\textwidth]{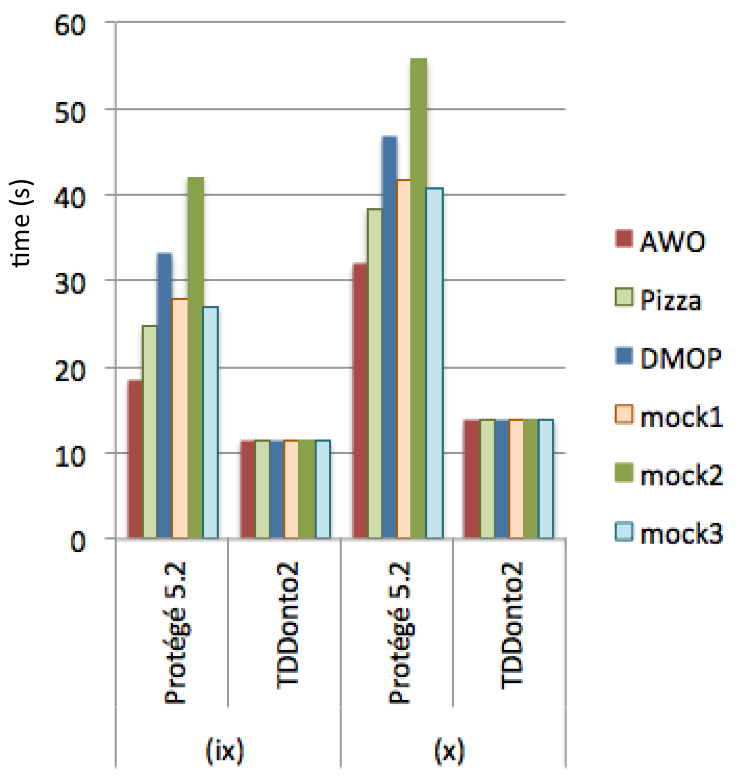}
\caption{Scenario (a) for axiom types (ix) and (x), i.e., where not the class expression editor is used (as in Figure~\ref{fig:clickdetails}), but the browsing and clicking interface instead. Note the different range of values on the y-axis cf. Figure~\ref{fig:clickdetails}.}\label{fig:clickdetails9and10}
\end{figure}

While the test-first is a clear winner, the difference is small for very small ontologies and the interface itself seems to have a bigger impact. Therefore, let us disaggregate the click-and-keystroke time by type of axiom. Then, the results are still favourable for TDDonto2---i.e., less time---except for a small difference for axiom type (ix) and (x); see Figure~\ref{fig:clickdetails}. 
We carried out a statistical analysis for each axiom type, where the null-hypothesis is that there is no difference. As the TDDonto2 series' data are not normally distributed, we used Wilcoxon for paired data, 2-tailed, and significance level 0.05. This was calculated over the data of the 4 scenarios pooled together so as not to cherry-pick (and 6 data points is insufficient for Wilcoxon). It was statistically significant for all axiom types except for type 5 (see online supplementary data).

There are several other noteworthy observations from this interface-only data. First, unlike with Prot\'eg\'e, the editing efficiency with TDDonto2 is immune to the vocabulary size names and hierarchy depths. This is partially thanks to the autocomplete feature. Second, there is an obvious gradient in the Prot\'eg\'e data. Given the three mock ontologies that varied class depth and vocabulary name size, it demonstrates that the hierarchy depth has the most negative effect, given mock2's $b_C=12$ and lower vocabulary name size cf DMOP, yet having higher values overall, and cf. mock1 and mock3 that have similar results but only the low hierarchy depth value remained the same.

Note also that the values for axiom types (ix) and (x) are similar and now also invariant for Prot\'eg\'e, because of the possible options to add it to the ontology and the task execution scenario selected was that of the class expression editor rather than firing the steps for axiom type (ii) twice and related auxiliary clicks (see Appendix for details). That is, Prot\'eg\'e now also requires the user to type the characters of the axiom, which is similar to the TDDonto2 interface. Using their respective scenario (a) with clicking, then the differences with TDDonto2 are the largest, as shown in Figure~\ref{fig:clickdetails9and10}, for it is a compounding effect adding up from the simpler axiom types. As they rely on axiom type (ii), which was already statistically significantly faster with TDDonto2, then so it is for axiom types (ix) and (x).

The comparison with other scenarios that vary click and keystroke times and autocomplete indicates robustness of results. That is, in the first alternate scenario (no autocomplete), Prot\'eg\'e has slightly lower values for most axiom types. In the second  and third scenarios (slower clicks and less autocomplete, respectively), TDDonto2 has lower values for most axiom types and they show a similar pattern as in Figure~\ref{fig:clickdetails} where axiom types (i)-(vii) are better for TDDonto2 and (ix)-(x) slightly better with Prot\'eg\'e (data not included). Thus, on the micro-level of the axiom, there is some editing difference between Prot\'eg\'e 5.2 and TDDonto2 for (very) small ontologies and an expert user, which becomes larger in favour of TDDonto2 for medium-sized and large ontologies. The latter is even more pronounced when taking into account the reasoner.

\subsubsection{Discussion}

It may be clear from the materials and methods section that there can be many parameters to assess editing efficiency,
that, perhaps, detracts one from the core result of test-first vs. test-last with the reasoner---the former having a higher editing efficiency. For the interface interaction,  
we made choices that seem reasonable to us, such as the autocomplete feature and the average case only 
and, to some extent, which of the alternate interface interaction scenarios was selected to be included\footnote{e.g., for (ix) and (x), the  clicking option could have been selected cf. the class expression editor, which resulted in the large difference with TDDonto2, as shown in Figure~\ref{fig:clickdetails9and10}. Besides that the authors would choose the latter option over firing the steps for axiom type (ii) twice, the class expression editor is the only option in case of a disjunction on the right-hand side, and therefore the class expression editor had been selected upfront.}. 
This serves the investigation into simulating a set of expert users in tool evaluations as well as teasing out parameters for non-human quantitative evaluations of ontology authoring tasks. In the case of the common Prot\'eg\'e 5.2 interface vs TDDonto2, the latter generally comes out favourably in several scenarios. 
This is especially so for medium and large-sized ontologies, which is further amplified with greater class hierarchy depth. 
In addition, if a ratio of axiom types would have been selected to be included in the calculations, rather than one of each, it would be mostly of type (ii), where TDDonto2 is distinctly more efficient, so this beneficial effect would thus be amplified further. 
Thus, ontologies within OWL 2 EL expressiveness characteristics (tailored to large, `simple' TBoxes), such as a SNOMED CT, would benefit most from TDD and its axiom-level input in TDDonto2.
This would be even more so for the scenario where new vocabulary has to be added, for the additional typing has a comparatively larger effect on the predominantly clicking-based Prot\'eg\'e interface with the existing vocabulary scenario we opted for. Finally, one may argue that most ontology developers are not expert efficient users and the editing efficiency calculations are artificial. Therefore, we shall return to this factor in the next section with the user evaluation.

The set-up has been lenient on the automated reasoner. This is because the aim was to obtain a general sense of any possible benefit in the authoring process, rather than aiming for improvements in the seconds. Practically, reasoning time may increase greatly as a result of adding an axiom. A quantitative approach would amount to randomised adding of axioms, so it would be an unpredictable effect. Reasoner performance {\em per s\'e} is not the scope of the paper, however, and therefore we deemed the approximation of 9x the baseline acceptable, rather than adding more than the baseline. Further, the mode of approximation taken is  favourable for the test-last setting rather than TDD, for the former requires more often the invocation of the reasoner. That TDD emerged positively already suggests it will be even more so in praxis.

\subsection{User study}
\label{sec:userstudy}

The editing efficiency evaluation reported on in the previous section assumed an optimal user working on an average ontology. Here, we are evaluating actual novice users on a relatively small ontology. We run this user study to test the following claims:
\begin{itemize}
\item[\textbf{C1}] Users will complete a \textbf{larger part of the task} when using the test-first TDDonto2 Prot\'eg\'e plugin than when using the test-last basic ontology editor interface  of Prot\'eg\'e.
\item[\textbf{C2}] Users will make \textbf{fewer mistakes} when using the TDDonto2 Prot\'eg\'e plugin than when using the basic ontology editor interface.
\item[\textbf{C3}] Users will be able to complete the tasks in \textbf{less time} when using the TDDonto2 plugin than when using the basic ontology editor interface.
\end{itemize} 

\subsubsection{Materials and methods}

The study included two general subtasks (assessing TBox axioms and ABox axioms) in one domain: African wildlife.  

The subjects were master students studying computer science, who have basic knowledge of ontologies and semantic technologies.
The users had been informed that the purpose of the user study was to evaluate using a tool for introducing knowledge into an ontology, a plugin to the Prot\'eg\'e editor TDDonto2, compared to using the existing Prot\'eg\'e interface without the plugin.
The users were presented with a demo of TDDonto2\footnote{accessible via \url{https://github.com/kierendavies/tddonto2}} 
and an ontology to be extended.

The general task was to assess the status of each axiom from the two sets (TBox axioms and ABox axioms) with respect to the ontology. 
The subjects could select the status from the following set: $\{$\textsf{entailed}, \textsf{absent}, \textsf{incoherent}, \textsf{inconsistent}$\}$, and they were provided textual definitions of each of the statuses.  We measured what percent of the statuses of the axioms the subjects were able to assess within given time (completeness) and what percent of their assessments were correct (correctness).

The experiment procedure was as follows:

\emph{For the given set of axioms, follow these steps:
\begin{compactitem}[-]
\item Register the current time (in the field provided).
\item Select the status of the given axiom (``Entailed'', ``Absent'', ``Incoherent'', ``Inconsistent'') for each axiom from the set,
\item Enter the axioms for which you have marked ``Absent" into the ontology and save the ontology file.
\item Copy the ontology source file to the given field.
\item Register the current time (in the field provided). 
\end{compactitem}
}
The experiment was designed as a \emph{within subject} study. 
The subjects were divided into two groups corresponding to the attendees of two class labs. 
There were 25 subjects altogether: 13 in the first group (which we denote Group A), and 12 in the second group (which we denote Group B).  
They were assigned four tasks:
\begin{compactitem}
\item Task $T1$: Testing the introduction of axioms regarding classes --  without TDDonto2 plugin,
\item Task $A1$: Testing the introduction of axioms regarding instances --  without TDDonto2 plugin,
\item Task $T2$: Testing the introduction of axioms regarding classes -- with TDDonto2 plugin,
\item Task $A2$: Testing the introduction of axioms regarding instances --  with TDDonto2 plugin.
\end{compactitem}

In Group A, the subjects first used the basic editor interface to complete the tasks and then TDDonto2. 
In Group B, the subjects first used TDDonto2 to complete the tasks and then the basic editor interface.
The time to complete tasks $T1$ and $T2$ (TBox axioms) was limited to 20 minutes per each task, and the time to complete tasks $A1$ and $A2$ (ABox axioms) was limited to 5 minutes for each task. 
Furthermore, to avoid the issue of a transfer effect (i.e., not repeating the errors the second time the subjects do the task) we prepared two different but comparable axiom sets\footnote{The tested axioms are at aforementioned URL (fn.~10)}.

\subsubsection{Results and discussion}

Table~\ref{tab:studystatistics} shows the data and statistics on the results of the user study, corresponding to claims C1-C3, with a breakdown by groups and the interfaces used. 
Regarding claim C1 (see Table~\ref{tab:studystatistics}a), we can see that, on average, more subjects completed the tasks when using TDDonto2 (94\% of task completeness versus 90\% regarding the basic interface), which supports our claim.
Regarding claim C2 (see Table~\ref{tab:studystatistics}b), relatively more subjects (88\%) correctly completed the tasks when using TDDonto2, while the percent of correct answers was lower in case of using the basic editor (62\%), which supports claim C2.
Regarding claim C3 (see Table~\ref{tab:studystatistics}c), we can see that the average time of completing the task was more than 2 times shorter when using TDDonto2 than when using the basic editor, more precisely it constituted 44\% of the time of completing the tasks when using the basic editor, which supports our claim C3.

\begin{table}[tb]
\caption{Statistics corresponding to claims C1, C2, C3. \label{tab:studystatistics}}
    \begin{subtable}[h]{1.0\columnwidth}
        \centering
\small
\begin{center}
\caption{Degree of completing the tasks per each group.}
\begin{tabular}{|c|p{1.8cm}|p{1.8cm}|p{1.5cm}|} 
\hline
Task & Group A  & Group B  & Total\\
 & (basic editor first) & (TDDonto2 first) & \\
\hline
Basic editor & & &\\
\hline
TBox1  & 13 (100\%)  & 10 (83\%) &\\
ABox1  & 13 (100\%)  & 9 (75\%) & 90\%\\
\hline
TDDonto2 & & &\\
\hline
TBox2  & 13 (100\%)  & 12 (100\%) &\\  
ABox2 & 12 (92\%)  & 10 (83\%) & 94\%\\
\hline
Total & 98\%  & 85\% & \\
\hline
\end{tabular}
\end{center}
    \end{subtable}
    \vspace{2mm}

    \begin{subtable}[h]{1.0\columnwidth}
        \centering
\small
\begin{center}
\caption{Correctness of completing the tasks per each group. 
}
\begin{tabular}{|c|p{1.8cm}|p{1.8cm}|p{1.5cm}|} 
\hline
Task & Group A  & Group B  & Total \\
 & (basic editor first) & (TDDonto2 first) & \\
\hline
Basic editor & & & \\
\hline
TBox1  & 61/117 (52\%)  & 63/90 (70\%) &  \\
ABox1  &  26/39 (67\%)  & 20/27 (74\%) & 62\%\\
\hline
TDDonto2 & & &\\
\hline
TBox2  &  98/117 (84\%)  & 96/108 (89\%) & \\  
ABox2  &  31/36 (86\%)  &  30/30 (100\%) & 88\%\\
\hline
Total & 70\%  & 82\% & \\
\hline
\end{tabular}
\end{center}

    \end{subtable}
        \vspace{2mm}

    \begin{subtable}[h]{1.0\columnwidth}
        \centering
\small
\begin{center}
\caption{Average time (in minutes) of completing the tasks per each group.}
\begin{tabular}{|c|p{1.8cm}|p{1.8cm}|p{1.5cm}|} 
\hline
Task & Group A  & Group B & Total \\
 & (basic editor first) & (TDDonto2 first) &\\
\hline
Basic editor & & &\\
\hline
TBox1  & 19.00 min & 19.19 min & \\
ABox1  &  4.92 min & 4.67 min & 12.27 min\\
\hline
TDDonto2 & & &\\
\hline
TBox2 &  8.15 min  & 7 min & \\  
ABox2 &  2.86 min  &  2.60 min & 5.38 min\\
\hline
Total &  8.97 min & 8.56 min& \\
\hline
\end{tabular}
\end{center}
    \end{subtable}
\end{table}

Figure~\ref{fig:corrperstatus} shows the comparison of correctness of completing the tasks when using the basic interface, and TDDonto2, disaggregated by the types of axiom statuses.

\begin{figure*}[t]
    \centering
    \begin{subfigure}[th]{0.45\textwidth}
        \includegraphics[width=0.95\columnwidth]{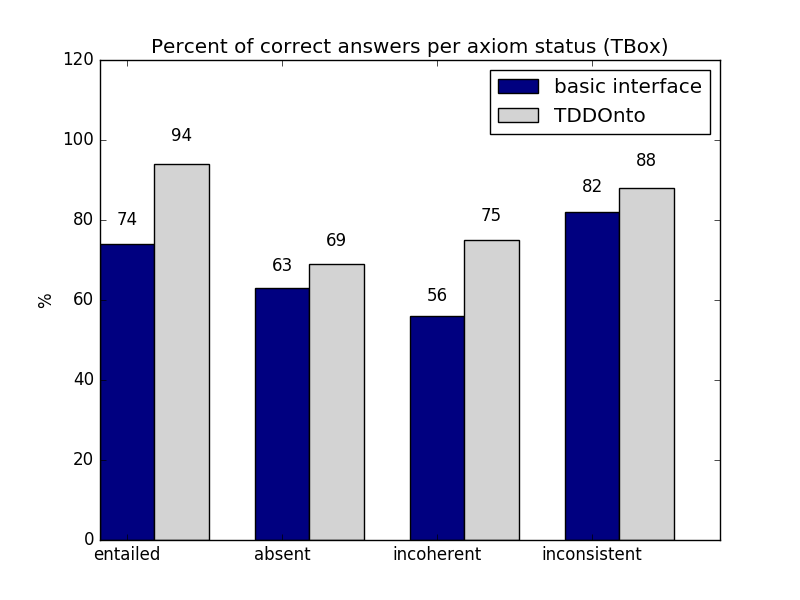}
        \caption{TBox}
    \end{subfigure}%
    \begin{subfigure}[th]{0.45\textwidth}
        \includegraphics[width=0.95\columnwidth]{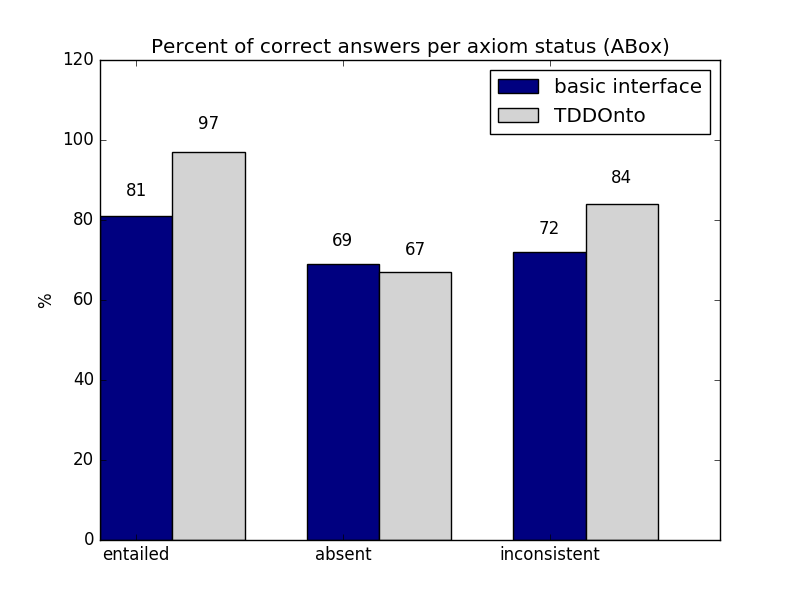}
        \caption{ABox}
    \end{subfigure}
    \caption{A breakdown of the correctness results---basic interface vs TDDonto2---with respect to status types. \label{fig:corrperstatus}}
\end{figure*}

We have performed a t-test for statistical significance of correctness percentage results.
We generated two sets for three settings (per TBox, per ABox, and per TBox plus ABox): a set with overall correctness percentage per each user per each task for the basic interface and a set with overall correctness percentage per each user per each task for TDDonto2.  
The null hypothesis was that relevant two sets of correctness percentages (the results for the basic interface and for TDDonto2) had identical mean (expected) values.
The results for the TBox experiments are: statistic=-3.3497, $p$=0.0015.
The results for the ABox experiments are: statistic=-2.8108, $p$=0.0072.
The overall result (TBox plus ABox) is as follows: statistic=-4.3337, $p$=3.55e-05.
Since the value of $p$ in all the cases is below 0.05, we reject the null hypothesis and conclude that the difference in the correctness results between the basic interface version of the experiments and the version when the subjects used TDDonto2 is statistically significant.

Finally, note that in the editing efficiency evaluation, editing in the AWO was about the same for Prot\'eg\'e and TDDonto2, for an assumed efficient expert user. Yet, these results with actual users demonstrate clearly that in praxis one can already observe TDD benefits even in these settings of small ontologies already.

\section{Discussion}
\label{sec:disc}

TDD as test-first approach to ontology authoring has been shown to be theoretically and technologically a viable option, and the first user study indicated that it is also beneficial for the authoring process from a user perspective. This clearly can be embedded in a broader process of ontology engineering, as the proposed lifecycle in \cite{KL16} already suggested. This can be extended further to also include goal or behaviour-driven development, which \textsc{Scone} \cite{Scone:Bitbucket} aims at, and conversion of competency questions into axioms that would feed into the technical TDD component presented in this paper by, e.g., linking it  to \cite{Dennis17}. The TDD component of regression testing---verifying past tests still pass---also may be an avenue for future works. Overall, these additions change the general flow of a TDD test of Figure~\ref{fig:tddontoIdea} into the one shown in Figure~\ref{fig:tddonto2}. Also in this case, more scenarios are possible than shown, so as not to obscure the general idea. For instance, after resolving conflicts when a precondition fails, one may not want the axiom in the ontology anymore. The current version of the TDDonto2 tool can cater for these variants, but it has no explicit interface features for them at present and it is left to the modeller's decisions.

\begin{figure}[t]
\includegraphics[width=0.48\textwidth]{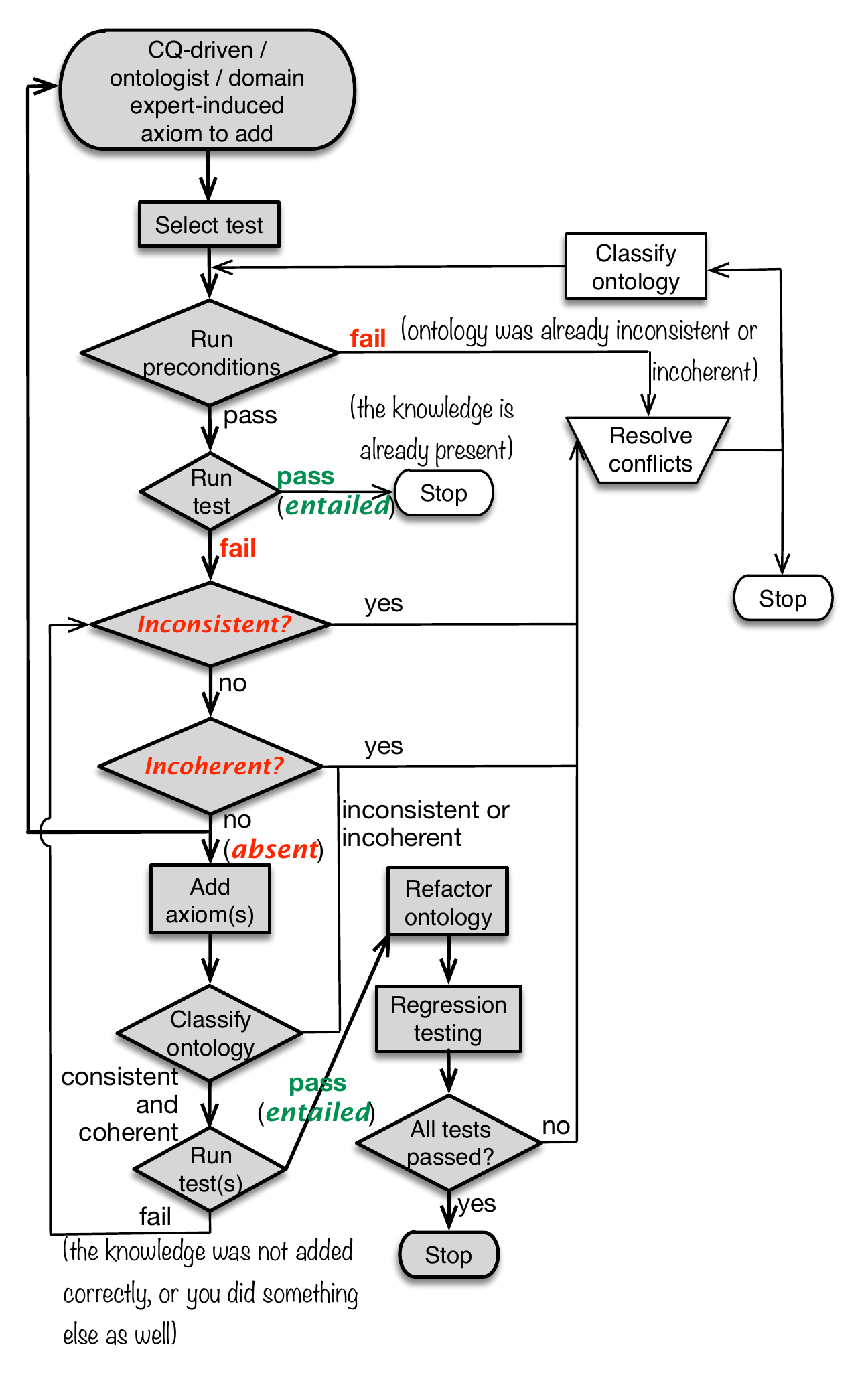}
\caption{Main steps of the general flow for the base case for a single TDD test, incorporating the model for testing.}\label{fig:tddonto2}
\end{figure}

Concerning theoretical and feature advances, the algorithms presented in Section~\ref{sec:algorithms} are the first ones with a broad coverage of OWL 2 language features, superseding those presented in all related work \cite{KL16,Scone:Bitbucket,Warrender15} especially on GCIs and the ABox, and also proving correctness of encoding. In addition, the model of testing axioms goes beyond the pass/fail/unknown and reporting missing vocabulary of the related work (which it does, too), by providing other possible outcomes that clarify what sort of a `fail' it is. This is a distinct feature for the setting of ontologies cf. software engineering, where a fail simply means ``not present, to implement'': the `fail/not present' may be because of absence due to lack of coverage, indeed, but also may be because adding it would cause inconsistency or incoherence, which is something one would want to know to determine the next step in the ontology authoring process. That is, unlike in software development, a `fail' does not necessarily imply `to add'.

We did make certain design decisions for this TDD that one may want to experiment with aside from the choice of technology\footnote{Note that alternatives to using the reasoner directly have been investigated, notably BGP with SPARQL-OWL \cite{Kollia11} and instance-based approach with mock objects, but exploiting the OWL reasoner turned out to be the fastest \cite{KL16,LK16}.}. For instance, once \textproc{isEntailed} is implemented by most or all reasoners, one could choose to update some of the algorithms accordingly. Also, one may also want to relax the coherency precondition. In our Model for Testing specification (Definition~\ref{def:testingmodel}), we sided with the somewhat `hardline' approach from a logician's viewpoint---a consistent theory, and every element satisfiable---compared to a possible tolerance for unsatisfiable classes at some point in the authoring stage. Anecdotally, we have seen behaviour along the line of  ``yes, I know x and y are inconsistent but I do not want to deal with them now''. Within the TDD scope, it would be preferable to remove them from the ontology, and add them at least temporarily as TDD test in the test set. This possibility eliminates issues with cascading unsatisfiable classes, yet not somehow losing that knowledge that with the test specification has become more easily examinable and thus resolvable.

The overall time of authoring an ontology is reduced thanks to not invoking the reasoner for each edit, which many a developer does \cite{Vigo14}, yet still being able to evaluate what the outcome would be if the axiom were to be added to the ontology. It does not reduce the reasoning time for the first classification, nor after actually having modified the ontology. Such efficiency improvements are reasoner improvements (e.g., using incremental reasoning), whereas here we focus on authoring improvements.

\section{Conclusions}
\label{sec:concl}

The novel test-driven development algorithms introduced in this paper fill a gap in rigour and coverage of both types of axioms that can be tested with a test-first approach and it provides more feedback to the modeller by means of its model of testing. The evaluation of this test-driven development in TDDonto2 with a novel human-independent assessment  approach for editing efficiency demonstrated that it is faster than the typical ontology authoring interface (Prot\'eg\'e 5.2) to some extent for smaller ontologies and even more so for medium to large ontologies with a stylised expert modeller, especially when automated reasoning is factored into the authoring process. Further, the user evaluation demonstrated that it is also more effective in task completion, time, and correctness (quality) for smaller ontologies and relative novice users. Thus, TDD's test-first approach with TDDonto2 is more effective than the common test-last authoring approach with Prot\'eg\'e.

The results demonstrate promise of test-driven development as an ontology development methodology. To turn it in a complete methodology, other components can be investigated, such as the refactoring step and the interaction with competency questions.

\subsection*{Acknowledgments} 
This work was partly supported by the Polish National Science Center (Grant No 2014/13/D/ST6/02076).



%







\section*{Appendix}



The calculations of the interface clicks for Prot\'eg\'e 5.2 are as follows. Given that several options are typically possible, we select one, which is indicated with an asterisk at the end of the option. noting that it offers several ways of adding an axiom in most cases. 
\begin{enumerate}[label=(\roman*)] 
\item 
click `Classes' tab (= 1), then either: 
\begin{enumerate}
\item drag class to position, if sufficiently nearby (existing classes) = 1
\item click class - click SubClass Of - in `class expression editor' type classname - click ok (existing classes) = $3 + a_C$ 
\item click class - click SubClass Of - in `class hierarchy' click as far down in the hierarchy as needed - select class - click ok (existing or new classes) = $4 + b_C$ \hfill [$^*$]
\end{enumerate}

\item 
click `Classes' tab (=1), then either:
\begin{enumerate}
\item click class - click SubClass Of - in `class expression editor' type ``R some/only D'' - click ok = $3 + a_{OP} + 4 + a_D$
\item click class - click SubClass Of - in `Object restriction creator' click as far down in the property hierarchy as needed - select property - in `Object restriction creator' click as far down in the restriction filler as needed - select class - click restriction type some - click ok = $5 + b_{OP} + b_C$ \hfill [$^*$]
\end{enumerate}

\item 
click `Classes' tab (=1), then either:
\begin{enumerate}
\item click class - click Disjoint With - in `class expression editor' type classname - click ok (existing classes) = $3 + a_C$
\item click class - click Disjoint With - in `class hierarchy' click as far down in the hierarchy as needed - select class - click ok = $4 + b_C$ \hfill [$^*$]
\end{enumerate}

\item 
click `Object properties' tab (=1), then either: 
\begin{enumerate}
\item click property - click Domain - in `class expression editor' type classname - click ok = $3 + a_C$
\item click property - click Domain - in `class hierarchy' click as far down in the hierarchy as needed - select class - click ok = $4 + b_C$ \hfill [$^*$]
\end{enumerate}

\item 
has the same processes as for (iv). 

\item 
click `Individuals by class' tab (=1), 
then either:
\begin{enumerate}
\item click Types - in `class expression editor' type classname - click ok = $2 + a_C$
\item click Types - in `class hierarchy' click as far down in the hierarchy as needed - select class - click ok = $3 + b_C$ \hfill [$^*$]
\item click `Classes' tab (=1), then 
 click Instances - click instance - click ok = 3
\end{enumerate}

\item 
click `Classes' tab (=1), then either:
\begin{enumerate}
\item click class - click SubClass Of - in `class expression editor' type ``R min x D'' - click ok = $3 + a_R + 3 + 1 + a_D = 7 + a_R + a_D$
\item click class - click SubClass Of - in `Object restriction creator' click as far down in the property hierarchy as needed - select property - in `Object restriction creator- click as far down in the restriction filler as needed - select class - click restriction type - click/type cardinality - click ok = $7 + b_{OP} + b_C$ \hfill [$^*$]
\end{enumerate}

\item 
click `Active ontology' tab (=1), then  
\begin{enumerate}
\item click `General class axioms' - click add - type the entire GCI - click ok = 3 + GCI
\end{enumerate}

\item 
click `Classes' tab (=1), then either:
\begin{enumerate}
\item execute the clicks for axiom type (ii) twice
\item click add - click class expression editor, and type: some $a_R$  $a_D$ and only $a_R$  $a_D$ = 2+  4 + $a_R$ + $a_C$ + 3 + 4 + $a_R$ + $a_D$ =  13 + $2* a_R$ + $2* a_D$ \hfill [$^*$]
\end{enumerate}

\item 
click `Classes' tab (=1), then either: 
\begin{enumerate}
\item execute the clicks for axiom type (ii), then click add - click class expression editor and type: some $a_S$ ($a_E$ or  $a_F$) = [axiom type (ii) clicks] + 2 +  4 + $a_S$ + 1+ $a_E$ + 2 + $a_F$ + 1 = [axiom type (ii) clicks] + 10 + $a_S$ + $a_E$ +  $a_F$
\item  click add - click class expression editor, and type: some $a_R$  $a_D$ and some $a_S$ ($a_E$ or  $a_F$) = 2 + 4 + $a_R$ + $a_D$ + 3 + 4 + $a_S$ + 1+ $a_E$ + 2 + $a_F$ + 1 = 16 + $a_R$ + $a_D$ + $a_S$ +  $a_E$ + $a_F$ \hfill [$^*$]
\end{enumerate}
\end{enumerate}

The clicks formulae with TDDonto2 are as follows. In the TDDonto2 plugin, one only adds full GCIs/assertions and then the user has to click ``Add'' (=1), %
and then, for the 10 axiom types, in the same order:  
\begin{enumerate}[label=(\roman*)]  
\item $a_C$ SubClassOf: $a_D = a_C + 11 + a_D  = 11 + a_C + a_D$
\item $a_C$ SubClassOf: some $a_R$  $a_D = a_C + 11 + 4 + a_R + a_D = 15 + a_C + a_R + a_D$
\item $a_C$ and $D$ SubClassOf: owl:nothing {\em or} $a_C$ SubClassOf: not $a_D = a_C + 3 +  a_D + 11 + 11 =  a_C + a_D + 22$ or $a_C + 11 + 3 + a_D = a_C + a_D + 14$
\item some $a_R$ SubClassOf: $a_C  = 4 + a_R + 11 + a_C = 15 + a_R + a_C$
\item some (inverse($a_R$) SubClassOf: $a_D = 4 + 9 + a_R + 11 + a_D = 24 + a_R + a_D$ 
\item c1 Type: C = $c + 5 + a_C$
\item $a_C$ SubClassOf: $a_R$ min n $a_D  = a_C + 11 + a_R + 4 + a_D = a_C + 15 + a_R + a_D$
\item non-simple class on the lhs (other than domain and range axiom) = can be anything
\item $a_C$ SubClassOf: some $a_R$  $a_D$ and only $a_R$  $a_D$ = $a_C + 11 + 4 + a_R + a_D + 3 + 4 + a_R + a_D =  22 + a_C + 2*a_R + 2*a_D$
\item $a_C$ SubClassOf: some $a_R$  $a_D$ and some $a_S$ ($a_E$ or  $a_F$) = $a_C$ + 11  + 4 + $a_R$ + $a_D$ + 3 + 4 + $a_S$ + 1+ $a_E$ + 2 + $a_F$ + 1 = 26 + $a_C$ + $a_R$ + $a_D$ + $a_S$ +  $a_E$ + $a_F$.
\end{enumerate}

\end{document}